\documentclass{article} 
\usepackage{iclr2026_conference,times}

\usepackage{amsmath,amsfonts,bm}

\def\eqref#1{equation~\ref{#1}}

\def\1{\bm{1}}

\DeclareMathAlphabet{\mathsfit}{\encodingdefault}{\sfdefault}{m}{sl}
\SetMathAlphabet{\mathsfit}{bold}{\encodingdefault}{\sfdefault}{bx}{n}

\usepackage[utf8]{inputenc} 
\usepackage[T1]{fontenc}    
\usepackage{hyperref}       
\usepackage{url}            
\usepackage{nicefrac}       
\usepackage{microtype}      
\usepackage{xcolor}         

\usepackage[table,x11names]{xcolor}   
\usepackage{graphicx}
\usepackage[table]{xcolor}      
\usepackage{tcolorbox}
\tcbuselibrary{listings, breakable, skins, theorems}
\usepackage{pgfplots}
\usepackage{times}
\usepackage{latexsym}
\usepackage{threeparttable}
\usepackage{xspace}
\usepackage{inconsolata}
\usepackage{pifont}
\usepackage{wrapfig}
\usepackage{multicol}
\usepackage{makecell}
\usepackage{tabularx}
\usepackage[rightcaption]{sidecap}
\usepackage{siunitx}
\usepackage{algorithm}
\usepackage{algpseudocode}
\usepackage{subcaption}
\usepackage{arydshln}
\usepackage{amsmath}   
\usepackage{amssymb}   
\usepackage{bm}        
\usepackage{graphicx}  
\usepackage{booktabs}  
\usepackage{dsfont}
\usepackage{amsfonts}       
\usepackage{enumitem}      
\usepackage{multirow}
\usepackage{colortbl}
\usepackage{array}
\usepackage{caption}
\usepackage{adjustbox}

\definecolor{oursrow}{HTML}{EBDBBC}
\definecolor{grayrow}{HTML}{BFBFBA}

\newcommand{\ie}{\emph{i.e.,}\xspace}
\newcommand{\eg}{\emph{e.g.,}\xspace}

\usepackage{amsthm}        
\newtheorem{assumption}{Assumption}

\newtheorem{theorem}{Theorem}
\newtheorem{lemma}{Lemma}
\newtheorem{proposition}{Proposition}
\theoremstyle{definition}

\newcommand{\ours}{\textsc{MMR1}\xspace}

\title{MMR1: Enhancing Multimodal Reasoning with Variance-Aware Sampling and Open Resources}

\author{Sicong Leng$^{1,2,*}$ \quad Jing Wang$^{1,*}$ \quad Jiaxi Li$^{3,*}$ \quad Hao Zhang$^{2,*}$ \quad Zhiqiang Hu$^{2}$\\[2pt]
\textbf{Boqiang Zhang}$^{2}$ \quad \textbf{Yuming Jiang}$^{2}$ \quad \textbf{Hang Zhang}$^{2}$ \quad \textbf{Xin Li}$^{2}$ \quad \textbf{Lidong Bing}$^{2}$\\[2pt]
\textbf{Deli Zhao}$^{2}$ \quad \textbf{Wei Lu}$^{1}$ \quad \textbf{Yu Rong}$^{2}$ \quad \textbf{Aixin Sun}$^{1,\dagger}$ \quad \textbf{Shijian Lu}$^{1,\dagger}$ \\[5pt]
$^{1}$Nanyang Technological University \\
$^{2}$DAMO Academy, Alibaba Group \\
$^{3}$Singapore University of Technology and Design \\[5pt]
$^{*}$Equal Contributions \quad $^{\dagger}$Correspondence
}

\iclrfinalcopy 
\begin{document}

\maketitle

\begin{abstract}
Large multimodal reasoning models have achieved rapid progress, but their advancement is constrained by two major limitations: the absence of open, large-scale, high-quality long chain-of-thought (CoT) data, and the instability of reinforcement learning (RL) algorithms in post-training. 
Group Relative Policy Optimization (GRPO), the standard framework for RL fine-tuning, is prone to \textit{gradient vanishing} when reward variance is low, which weakens optimization signals and impairs convergence.
This work makes three contributions: (1) We propose Variance-Aware Sampling (VAS), a data selection strategy guided by Variance Promotion Score (VPS) that combines outcome variance and trajectory diversity to promote reward variance and stabilize policy optimization. (2) We release large-scale, carefully curated resources containing $\sim$1.6M long CoT cold-start data and $\sim$15k RL QA pairs, designed to ensure quality, difficulty, and diversity, along with a fully reproducible end-to-end training codebase. (3) We open-source a family of multimodal reasoning models in multiple scales, establishing standardized baselines for the community.
Experiments across mathematical reasoning benchmarks demonstrate the effectiveness of both the curated data and the proposed VAS. Comprehensive ablation studies and analyses provide further insight into the contributions of each component. In addition, we theoretically establish that reward variance lower-bounds the expected policy gradient magnitude, with VAS serving as a practical mechanism to realize this guarantee.
Our code, data, and checkpoints are available at \url{https://github.com/LengSicong/MMR1}.
\end{abstract}

\section{Introduction}\label{sec:intro}
Recent advances in large language and multimodal reasoning models have markedly improved performance on complex tasks such as mathematics, science, and open-domain problem solving. 
Reinforcement learning (RL) plays a central role in these developments by optimizing models with process- or outcome-based rewards~\citep{lightman2024lets,wang2024math,Li2025FromS1}. 
Group Relative Policy Optimization (GRPO; \citet{shao2024deepseekmath}) has emerged as a widely adopted RL framework due to its efficiency and scalability, and has been successfully applied to both language models~\citep{deepseekai2025deepseekr1} and multimodal models~\citep{meng2025mmeureka,wang2025vlrethinker,tan2025reasonrft,MMR1-Math2025}. 
However, GRPO is inherently susceptible to \textit{gradient vanishing}: when sampled rewards have low variance, relative advantages collapse toward zero, weakening optimization signals and destabilizing training~\citep{razin2024vanishing,razin2025what}. 
This issue persists across both unimodal and multimodal contexts, posing a fundamental challenge to effective RL optimization.

In parallel, the progress of multimodal reasoning research is influenced by the limited availability of open, large-scale, high-quality long chain-of-thought (CoT) data. 
Compared with text-only reasoning, where multiple datasets are publicly accessible~\citep{guha2025openthoughts,muennighoff2025s1}, multimodal training often relies on more restricted resources, which may constrain reproducibility and further development.
Recent studies have made progress by exploring heuristic data curation~\citep{meng2025mmeureka,huang2025visionr1,chen2025geopqa}, reward design modifications~\citep{tan2025reasonrft,shen2025vlmr1}, and training adjustments~\citep{deng2025openvlthinker,zhang2025r1vl}.
While these approaches improve downstream performance, challenges related to stable GRPO optimization and the broader availability of curated multimodal reasoning data remain underexplored.

\begin{figure}[t]
    \centering
    \includegraphics[width=0.9\linewidth]{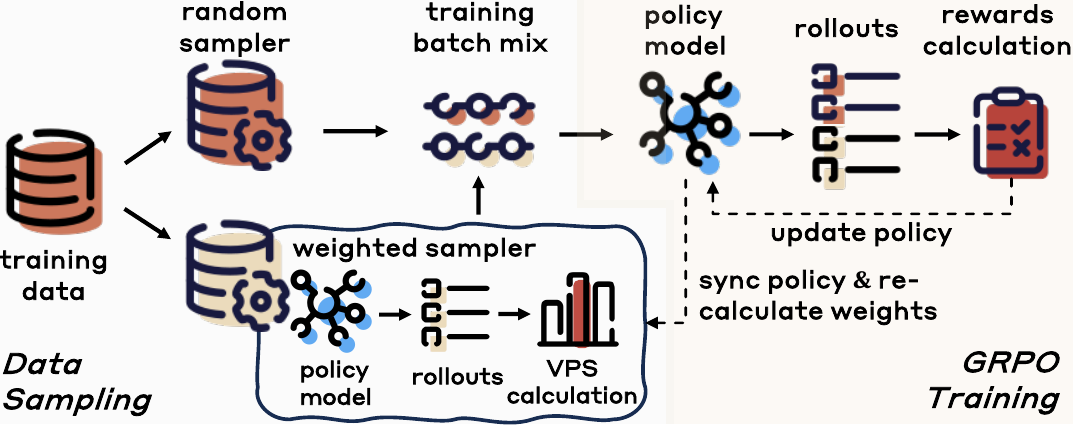}
    \caption{\textbf{Overview of the Variance-Aware Sampling (VAS) framework.}}
    \label{fig:main}
\end{figure}

In this work, we introduce \textbf{V}ariance-\textbf{A}ware \textbf{S}ampling (\textbf{VAS}), a dynamic data selection strategy designed to mitigate gradient vanishing in GRPO-based training for multimodal reasoning models. 
Our approach is grounded in the theoretical insight that \textit{reward variance provides a lower bound on the expected policy gradient magnitude}. Increasing reward variance, therefore, offers a principled means to stabilize training and strengthen policy optimization.
Specifically, VAS employs the \textbf{V}ariance \textbf{P}romotion \textbf{S}core (\textbf{VPS}), which evaluates each prompt's potential to induce reward variance. 
VPS consists of two complementary components: the \textbf{O}utcome \textbf{V}ariance \textbf{S}core (\textbf{OVS}), which favors prompts yielding a balanced mix of correct and incorrect responses to maximize expected reward variance, and the \textbf{T}rajectory \textbf{D}iversity \textbf{S}core (\textbf{TDS}), which encourages diversity among reasoning trajectories, thereby raising the lower bound of variance and sustaining informative gradient signals even under sparse or noisy correctness feedback.
Depicted in Figure~\ref{fig:main}, VAS constructs each training batch from two subsets: one sampled with probabilities proportional to VPS, emphasizing prompts with higher potential to induce reward variance, and another drawn uniformly at random to maintain broad data coverage. This design aims to balance targeted promotion of reward variance with general exposure to the training distribution. By introducing outcome- and trajectory-level variability into GRPO’s group-based comparisons, VAS reduces the risk of weak gradients and contributes to more stable optimization.
A theoretical analysis is presented in \S\ref{sec:theory}.

We validate VAS on a range of multimodal mathematical and logical reasoning benchmarks.
Experiments show that VAS improves convergence, stability, and downstream performance. 
Ablation studies further demonstrate that OVS and TDS contribute complementary benefits: OVS enhances expected reward variance by balancing outcomes, while TDS increases trajectory diversity to support more consistent gradient updates. 
Beyond methodology, we curate and release large-scale datasets for both supervised fine-tuning and RL. The supervised dataset ($\sim$1.6M) emphasizes long chain-of-thought reasoning paired with verified short answers, while the RL dataset ($\sim$15k) is constructed to capture diverse levels of difficulty and domain coverage. Both datasets are curated with explicit attention to quality, difficulty, and diversity, ensuring their value for multimodal reasoning research. Together with these resources, we provide a reproducible codebase and open models at multiple scales, offering standardized baselines for future research.

\section{Related Work}\label{sec:related_work}

Building on the success of rule-based RL~\citep{deepseekai2025deepseekr1,kimiteam2025kimik15}, recent multimodal work explores RL with verifiable rewards, typically following a pipeline that conducts optional SFT activation then applies RL~\citep{schulman2017ppo,ahmadian2024back}, such as GRPO~\citep{shao2024deepseekmath}, with fine-grained training recipes. 
In the multi-modal domain, various approaches refine this pipeline through specific reward design~\citep{tan2025reasonrft,shen2025vlmr1}, sample diversification~\citep{meng2025mmeureka,wang2025vlrethinker}, hyperparameter tuning~\citep{huang2025visionr1,tan2025reasonrft,yang2025r1onevision}, and advanced RL strategies~\citep{peng2025lmmr1,deng2025openvlthinker,zhang2025r1vl}.
Nevertheless, they often overlook the gradient vanishing problem inherent in GRPO-based training, resulting in unstable optimization and slow convergence~\citep{razin2025what}. 
Some studies attempt to alleviate it by filtering samples with moderate pass rates~\citep{wang2025vlrethinker,meng2025mmeureka}, yet they remain largely heuristic and lack comprehensive experimental validation or theoretical grounding.

Recent studies have examined gradient vanishing in a principled manner, analyzing how training objectives degrade under reward sparsity, variance reduction, and optimization dynamics~\citep{liu2025understanding,hu2025reinforce++,vassoyan2025ignore,zhou2025q}. A range of remedies has been proposed, including reward rescaling~\citep{li2024remax,huang2025lean}, entropy regularization~\citep{liu2024improving}, and improved sample selection~\citep{wang2024cpl,zhang2025focused,li2025limr}, which primarily operate by adjusting RL algorithms or reward mechanisms. In contrast, our work takes an orthogonal perspective by mitigating gradient vanishing through data sampling during training. Supported by both theoretical grounding and extensive empirical validation, our approach complements existing GRPO variants~\citep{wang2024offline,liu2024improving,hu2025reinforce++,Yue2025VAPOEA} and can be naturally combined with them to further improve training stability and effectiveness in reasoning.

\section{Variance-Aware Sampling Framework}\label{sec:method}

This section details the Variance-Aware Sampling (VAS) framework. \S\ref{sec:score_definition} defines the Variance Promotion Score (VPS), comprising Outcome Variance Score (OVS) and Trajectory Diversity Score (TDS), which guides dynamic data sampling. \S\ref{sec:sampler_implementation} then outlines the sampler implementation, including VPS updates and sample selection during training.

\subsection{Variance Promotion Score}
\label{sec:score_definition}
Let $x$ be a prompt with ground-truth answer $\bar{y}$ from the training set.
In GRPO framework, the model generates $N$ responses $\{y_i\}_{i=1}^N$ for each $x$. 
A task-specific verifier $\text{V}(x, y_i, \bar{y}) \in \{0,1\}$ evaluates each response, returning $1$ if $y_i$ matches $\bar{y}$ and $0$ otherwise.
The \textbf{pass rate} for $x$ is then defined as:
\[
\text{P}(x) = \frac{1}{N} \sum_{i=1}^N \text{V}(x, y_i, \bar{y}).
\]
Inspired by \citet{foster2025learning}, we directly calculate the Outcome Variance Score (OVS) as:
\[
\text{OVS}(x) = \text{P}(x)(1 - \text{P}(x)),
\]
which corresponds to the Bernoulli variance of correctness across responses. It is maximized at $\text{P}(x)=0.5$, where correct and incorrect outputs are balanced.
We further define the Trajectory Diversity Score (TDS) to characterize variability in reasoning processes. Let $\text{Diversity}(\{y_i\}_{i=1}^N)$ be a diversity function over sequences (\eg inverse self-BLEU or distinct-n):
\[
\text{TDS}(x) = \text{Diversity}(\{y_i\}_{i=1}^N),
\]
where a higher value reflects greater diversity among sampled trajectories.
The overall Variance Promotion Score (VPS) is computed as a weighted combination of OVS and TDS:
\[
\text{VPS}(x) = \alpha \cdot \text{OVS}(x) + \beta \cdot \text{TDS}(x),
\]
where $\alpha, \beta > 0$ balance their contributions. OVS increases the expected reward variance, while TDS provides a lower bound by encouraging trajectory diversity. Together, they are intended to strengthen the magnitude and consistency of gradient signals in GRPO training.

\subsection{Dynamic Sampler}
\label{sec:sampler_implementation}
The dynamic sampler prioritizes prompts with higher VPS, which are expected to induce greater reward variance during training.
At each sampling step, the training batch is constructed from two subsets: one from a weighted sampler based on VPS and another from a uniform random sampler.
A mix ratio hyperparameter $\lambda \in [0, 1]$ controls the proportion of samples drawn from each source, with $\lambda$ specifying the fraction of the batch selected from the weighted sampler.
VPS scores are periodically updated to reflect changes in the policy. After every $T_\text{update}$ steps, $N$ responses are resampled for each prompt, and OVS and TDS are recomputed. The update interval $T_\text{update}$ trades off computational cost against adaptation speed.
Algorithm~\ref{alg:vas_sampler} summarizes the VAS procedure.

\begin{algorithm}[h]
\caption{Variance-Aware Sampling (VAS) for GRPO Training}
\label{alg:vas_sampler}
\footnotesize
\begin{algorithmic}[1]
\Require Dataset $\mathcal{D}$; batch size $B$; rollouts per prompt $N$; VPS update interval $T_{\text{update}}$; mix ratio $\lambda\!\in\![0,1]$
\State \textbf{Initialize} policy parameters $\theta$
\For{each prompt $x \in \mathcal{D}$} \Comment{Initial VPS estimation}
    \State Sample $N$ rollouts $\{y_i\}_{i=1}^{N}$ from $\pi_\theta(\cdot\mid x)$
    \State Compute pass rate $P(x)$, $\mathrm{OVS}(x)$, $\mathrm{TDS}(x)$, and $\mathrm{VPS}(x)$
\EndFor
\For{training step $t=1,\dots,T$}
    \If{$t \bmod T_{\text{update}}=0$} \Comment{Periodic VPS refresh}
        \For{each prompt $x \in \mathcal{D}$}
            \State Sample $N$ rollouts $\{y_i\}_{i=1}^{N}$; update $P(x)$, $\mathrm{OVS}(x)$, $\mathrm{TDS}(x)$, $\mathrm{VPS}(x)$
        \EndFor
    \EndIf
    \State $B_{\mathrm{w}}\!\leftarrow\!\lfloor \lambda B \rfloor,\quad B_{\mathrm{r}}\!\leftarrow\! B-B_{\mathrm{w}}$
    \State Sample $B_{\mathrm{w}}$ prompts from $\mathcal{D}$ \emph{with replacement} proportional to $\mathrm{VPS}(\cdot)$
    \State Sample $B_{\mathrm{r}}$ prompts from $\mathcal{D}$ uniformly at random
    \State $\mathcal{B}\!\leftarrow$ union of the two sets \Comment{Construct training batch}
    \For{each $x \in \mathcal{B}$}
        \State Sample $N$ rollouts $\{y_i\}_{i=1}^{N}$ from $\pi_\theta(\cdot\mid x)$
        \State Compute GRPO loss $\mathcal{L}_{\mathrm{GRPO}}(x,\{y_i\})$ 
    \EndFor
    \State Update $\theta \leftarrow \theta - \eta \nabla_\theta \sum_{x \in \mathcal{B}} \mathcal{L}_{\mathrm{GRPO}}(x,\{y_i\})$
\EndFor
\end{algorithmic}
\end{algorithm}




\section{Theory}\label{sec:theory}

This section formalizes the intuition that prompts with higher \emph{reward variance} yield more informative policy-gradient updates. We first establish a \textit{Variance–Progress Theorem} for the vanilla REINFORCE algorithm~\citep{williams1992reinforce}, showing that expected improvement is linearly lower-bounded by the reward variance of the prompt. We then present a two-level decomposition of reward variance that directly motivates the \textit{Outcome Variance Score} and \textit{Trajectory Diversity Score} in our method. Finally, we extend the analysis to GRPO, with complete proofs and derivations provided in Appendix~\ref{apdx:sec:proofs}.

\subsection{Preliminaries}\label{sec:prelims}
Let $x\in\mathcal X$ be a prompt and $y\in\mathcal Y$ a response drawn from the policy $\pi_\theta(y\mid x)$. A learned reward model $R:\mathcal X\times\mathcal Y\to[-1,1]$ assigns a scalar reward. The optimization objective is defined as:
\begin{equation}
    J(\theta)=\mathbb{E}_{x\sim\mathcal D}\;
              \mathbb{E}_{y\sim\pi_\theta(\cdot\mid x)}
              \,[R(x,y)],
\end{equation}
where $\mathcal D$ is the prompt distribution. Here, we omit the KL regularization, as it only rescales constants and does not affect the variance argument.  

\paragraph{Score-function identity.}  
For a fixed prompt $x$, the gradient can be expressed as:
\begin{equation}
    \nabla_\theta J_x(\theta) =\mathbb{E}_{y\sim\pi_\theta}\bigl[R(x,y)\,g(x,y)\bigr],
    \qquad g(x,y)=\nabla_\theta\log\pi_\theta(y\mid x),
\end{equation}
where $g(x,y)$ is the score function and $\mathbb E_{y}[g(x,y)]=0$.  

\paragraph{Baselines and variance.}  
Subtracting a prompt-dependent baseline $b(x)$ keeps the estimator unbiased:
\begin{equation}
    G(x)=\mathbb{E}_{y}[g(x,y)(R(x,y)-b(x))],
    \qquad \mathbb E[G(x)]=\nabla_\theta J_x(\theta).
\end{equation}
The variance is minimized when $b^\star(x)=\mathbb E_{y}[R(x,y)]$. Under this choice, the variance factorizes as (Lemma A.1):
\begin{equation}
    \operatorname{Var}[G(x)]
        =\operatorname{Var}_{y}[R(x,y)]\;\Gamma_\theta(x),
\end{equation}
where $\Gamma_\theta(x)=\mathbb E_{y}\|g(x,y)\|^2$ is a Fisher-information term depending on the policy but not on the rewards. Proposition A.2 shows $\Gamma_\theta(x)$ is bounded above and below by model-dependent constants. Thus, the \emph{reward variance} is the only prompt-dependent factor controlling the dispersion of stochastic gradients.  

\subsection{A Variance–Progress Theorem for REINFORCE}\label{sec:thm}
Consider a single gradient step $\theta^{+}=\theta+\eta G(x)$ on prompt $x$, with learning rate $\eta>0$.  

\begin{assumption}[Smoothness and gradient lower bound]\label{as:smooth}
For each prompt $x$:  
(i) $J_x(\theta)$ is twice differentiable with $\|\nabla^2_\theta J_x(\theta)\|\le L$ for some $L>0$;  
(ii) There exists $c_{\min}>0$ such that $\|\nabla_\theta J_x(\theta)\|^2 \ge c_{\min}\,\operatorname{Var}_{y}[R(x,y)]$.  
\end{assumption}

The second condition links reward variance to the squared gradient norm. It follows from the positive definiteness of the Fisher information matrix under mild regularity conditions.  

\begin{theorem}[Variance–Progress]\label{thm:var-progress}
Let Assumption~\ref{as:smooth} hold and use the optimal baseline $b^\star(x)$. Then for any step size  
$0<\eta \le \tfrac{c_{\min}}{4L\Gamma_\theta(x)},$  
the expected one-step gain satisfies
\begin{equation}
    \mathbb E\bigl[J_x(\theta^{+})-J_x(\theta)\bigr]
    \;\ge\;\frac{\eta\,c_{\min}}{4}\,
    \operatorname{Var}_{y}[R(x,y)].
\end{equation}
\end{theorem}

\begin{tcolorbox}[colback=gray!5,colframe=gray!40,title=Intuition]
\footnotesize
Prompts with higher reward variance provide stronger gradient signals. 
If all rollouts have similar rewards, gradients vanish; if outcomes or reasoning paths vary, the variance ensures progress per update.
\end{tcolorbox}

\paragraph{Sketch of proof.}  
A second-order Taylor expansion yields $J_x(\theta^{+})=J_x(\theta)+\langle\nabla J_x,G(x)\rangle+\tfrac12 G^\top H G$. Taking expectations and using unbiasedness gives a linear term $\eta\|\nabla J_x\|^2$. The quadratic remainder is bounded by $\tfrac12\eta^2 L\,\mathbb E\|G\|^2$. Substituting the variance factorization and the gradient lower bound, and restricting $\eta$ as stated, ensures the remainder is at most half the linear term, giving the result. The full details of the proof are provided in Appendix~\ref{apdx:ssec:theorem-proof}.  

\subsection{Variance Decomposition and Connection to OVS/TDS}
For binary rewards, the variance admits a natural two-level decomposition.  
Write each rollout $y=(y_{\mathrm{cot}},y_{\mathrm{ans}})$ as a reasoning chain and its final answer, and let the verifier assign $R(x,y)=\mathbf 1_{\text{correct}(y_{\mathrm{ans}})}$. Define $Z=\varphi(y_{\mathrm{cot}})$ as a representation of the reasoning chain and $p_Z(x)=\Pr(R=1\mid Z)$. The law of total variance gives
\begin{equation}
    \operatorname{Var}_{y}[R(x,y)]
      = \underbrace{\mathbb E_{Z}[p_Z(x)(1-p_Z(x))]}_{\text{intra-trajectory}}
      + \underbrace{\operatorname{Var}_{Z}[p_Z(x)]}_{\text{inter-trajectory}}.
\end{equation}
The first term corresponds to variability in correctness given a fixed reasoning path, and is estimated by $\hat p(x)(1-\hat p(x))$. This motivates the \emph{Outcome Variance Score (OVS)}. The second term measures variation across reasoning paths, which can be lower-bounded by diversity metrics such as normalized edit distance or self-BLEU dispersion (detailed in Appendix~\ref{apdx:ssec:two-level-decomp}). This motivates the \emph{Trajectory Diversity Score (TDS)}. Together, OVS and TDS provide complementary mechanisms to raise reward variance, directly connecting the theory to our method in \S\ref{sec:method}.  

\subsection{From REINFORCE to GRPO}\label{sec:grpo-connection}
GRPO extends REINFORCE by normalizing rewards within each group of rollouts. For prompt $x$, the group mean reward $\bar r(x)$ serves as a baseline, and the sample standard deviation $s(x)$ whitens centered rewards:
\[
\tilde R(x,y_i) = \frac{R(x,y_i)-\bar r(x)}{s(x)+\delta}.
\]
The gradient estimator then multiplies $\tilde R$ by the importance ratio $r_\theta(y_i\mid x)$, optionally clipped to $1\pm\varepsilon$ to control KL divergence.  

Without clipping, the estimator remains unbiased, and Theorem~\ref{thm:var-progress} applies with a rescaled learning rate bound reflecting the whitening factor. With clipping, the estimator acquires a bias of order $O(\varepsilon)$, which reduces but does not eliminate the lower bound. Thus, under both settings, prompts with higher reward variance continue to guarantee larger provable minimum improvements per update. Details and finite-sample corrections are provided in Appendix~\ref{apdx:ssec:grpo}.

\section{Data Curation}\label{sec:data_curation}
Following prior work~\citep{tan2025reasonrft,wang2025vlrethinker,meng2025mmeureka}, our training pipeline includes two stages: a supervised cold-start stage followed by a reinforcement learning stage using GRPO, with each stage supported by meticulously curated datasets.

\definecolor{oursrow_orig}{HTML}{BFBFBA}
\colorlet{oursrow}{oursrow_orig!30}

\begin{table}[t]
\caption{Cold-start Data Statistics.} 
\label{tab:sft_data}
\centering
\renewcommand{\arraystretch}{1.2}
\setlength{\tabcolsep}{0.6mm}
\setlength{\aboverulesep}{0pt}
\setlength{\belowrulesep}{0pt}
\begin{tabularx}{\textwidth}{@{}cc>{\raggedright\arraybackslash}X@{}} 
\toprule
\textbf{Domain} & \textbf{\# Samples} & \textbf{ Datasets} \\ \midrule
Math & 1.6M & MM-MathInstruct~\citep{Wang2025MathCoderVLBV}, MathV360K~\citep{Shi2024MathLLaVABM}, MultiMath~\citep{Peng2024MultiMathBV}, MATH~\citep{Hendrycks2021MeasuringMP}, MathVision~\citep{Wang2024MeasuringMM}, CLEVR-Math~\citep{Lindstrm2022CLEVRMathAD}, IconQA~\citep{Lu2021IconQAAN}, MAVIS-Instruct~\citep{Zhang2024MAVISMV}, RCoT~\citep{Deng2024TheoremValidatedRC}, GeoQA+~\citep{Cao2022AnAB}, Geometry3K~\citep{Lu2021InterGPSIG}, GeomVerse~\citep{Kazemi2023GeomVerseAS}, {\color{gray}Self-collected data} \\  \hdashline
\rowcolor{oursrow}
Science & 13K & ScienceQA~\citep{Lu2022LearnTE}, AI2D~\citep{kembhavi2016diagram}, CLEVR~\citep{Johnson2016CLEVRAD}, {\color{gray}Self-collected data} \\  \hdashline
Chart-Figure & 8K & ChartQA~\citep{Masry2022ChartQAAB}, SPIQA~\citep{Pramanick2024SPIQAAD}, DVQA~\citep{Kafle2018DVQAUD}, PlotQA~\citep{Methani2019PlotQARO} \\  \hdashline
\rowcolor{oursrow}
Doc-Table & 8K & TableMWP~\citep{Lu2022DynamicPL}, InfoVQA~\citep{Mathew2021InfographicVQA}, DocVQA~\citep{Mathew2020DocVQAAD}, WikiTableQuestions~\citep{Pasupat2015CompositionalSP}, VisualMRC~\citep{Tanaka2021VisualMRCMR} \\  \hdashline
General & 8K & Sherlock~\citep{hesselhwang2022abduction}, A-OKVQA~\citep{Schwenk2022AOKVQAAB}, PISC~\citep{Li2017DualGlanceMF}, GQA~\citep{Hudson2019GQAAN} \\ 
\bottomrule
\end{tabularx}
\end{table}

\subsection{Cold-start Data}
\label{ssec:cold_start}

This stage utilizes long chain-of-thought (CoT) data for supervised fine-tuning. As summarized in Table~\ref{tab:sft_data}, we collect question-answer pairs from diverse instruction-tuning corpora, which primarily feature short answers or short CoTs. 
To address the limited coverage, particularly in science, in existing multimodal datasets, we supplement with publicly available practice problems and exams from biology, chemistry, geography, and physics.
The final dataset spans five domain categories: \textit{Math}, \textit{General}, \textit{Chart-Figure}, \textit{Doc-Table}, and \textit{Science}.
To ensure balanced quality and difficulty, we apply filtering based on response correctness. For each question, multiple responses are generated via Qwen2.5-VL-72B~\citep{bai2025qwen25vl} and verified by GPT-4o~\citep{openai2024gpt4o}. Samples are categorized by \textit{pass rate}: easy ($\geq 0.8$), hard ($\leq 0.2$), or medium (otherwise). Only medium and hard cases are retained for cold-start. Long CoT annotations are then produced using Gemini 2.5 Pro/Flash~\citep{google2025gemini25pro}, which generates multi-step rationales followed by final answers. Annotations are preserved only when the final answer matches ground truth as validated by GPT-4o.


\subsection{RL Data}
\label{ssec:rl_data}
This stage employs prompts annotated with \textit{concise}, \textit{verifiable} short answers suitable for reward modeling. Specifically, we adopt GPT-4o to extract and rephrase open-form short answers from the original CoT annotations. 
The RL dataset integrates \textbf{two} complementary components. 
First, we select $8$k math problems from the cold-start stage, retaining only hard-level items with low pass rates to ensure challenging supervision.
These problems are further categorized into fine-grained types with GPT-4o (detailed in Appendix~\ref{apdx:ssec:math_types}) and uniformly sampled to ensure balanced coverage across categories.
Second, we add $7$k logical reasoning problems from Raven~\citep{Zhang2019RAVENAD}, MM-IQ~\citep{Cai2025MMIQBH}, and EasyArc~\citep{Unsal2025EasyARCEV}, curated to incentivize general reasoning ability beyond math. 
Together, these components form a $15$k RL dataset emphasizing difficulty, diversity, and balanced coverage across mathematical and logical reasoning tasks.


\section{Experiments}

\begin{table}[t]
\centering                     
\renewcommand{\arraystretch}{1.2}  
\setlength{\tabcolsep}{1.8mm}      
\footnotesize         
\caption{Comparison of \ours with other MLLMs on mathematics-related benchmarks. All models are reevaluated under identical conditions for fairness; values in parentheses are taken from the original papers. The \textbf{bold} and \underline{underline} highlight the best and second-best scores, respectively.}
\adjustbox{max width=1.0\textwidth}{
\begin{tabular}{l c llllc c}
\toprule
\textbf{Model} & \textbf{Size} & \textbf{MathVerse} & \textbf{MathVista} & \textbf{MathVision} & \textbf{LogicVista} & \textbf{ChartQA} & \textbf{Avg} \\
\midrule
\rowcolor[HTML]{F5F5F5}
\multicolumn{8}{c}{General-Purpose Models}\\
\midrule
Qwen2.5-VL & 7B & 50.4 (49.2) & 69.3 (68.2) & 28.7 (25.1) & 44.0 & 82.4 & 55.0 \\
InternVL2.5 & 8B & 40.0 (39.5) & 61.4 (64.4) & 19.9 (19.7) & 37.7 (36.0) & 73.4 & 46.5 \\
InternVL3 & 8B & 49.4 (39.8) & 68.5 (71.6) & 30.0 (29.3) & 41.3 (44.1) & 81.3 & 54.1 \\
LLaVA-OV & 7B & 33.6 (26.2) & 56.4 (63.2) & 15.9 & 30.6 & 65.0 & 40.3 \\
\midrule
\rowcolor[HTML]{F5F5F5}
\multicolumn{8}{c}{Reasoning-Oriented Models}\\
\midrule
MM-Eureka & 8B & 52.3 (50.3) & 73.4 (73.0) & 29.4 (26.9) & \underline{47.1} & 82.7 & 57.0 \\
R1-VL & 7B & 41.3 (40.0) & 61.5 (63.5) & 23.0 (24.7) & 36.3 & 76.3 & 47.7 \\
R1-OneVision & 7B & 44.0 (46.4) & 60.3 (64.1) & 22.0 (29.9) & 40.0 & 72.5 & 47.8 \\
OpenVLThinker & 7B & 48.1 (47.9) & 70.6 (70.2) & 22.0 (25.3) & 41.0 & 81.0 & 52.5 \\
VL-Rethinker & 7B & \underline{54.6} (54.2) & \underline{73.7} (74.9) & 30.1 (32.3) & 45.7 & \underline{83.5} & 57.5 \\
Vision-R1 & 7B & 51.9 (52.4) & 72.1 (73.5) & -- & 44.7 & 82.7 & -- \\
ThinkLite-VL & 7B & 51.3 (50.7) & 72.5 (75.1) & 27.5 & 44.3 & 83.1 & 55.7 \\
VL-Cogito & 7B & 54.3 & \textbf{74.8} & \underline{30.7} & \textbf{48.9} & 83.4 & \underline{58.2} \\
\midrule
\rowcolor[HTML]{E8F5E9} \ours & 3B & 47.9 & 67.1 & 25.4 & 42.0 & 81.2 & 52.7 \\
\rowcolor[HTML]{E8F5E9} \ours & 7B & \textbf{55.4} & 72.0 & \textbf{31.8} & \textbf{48.9} & \textbf{83.7} & \textbf{58.4} \\
\bottomrule
\end{tabular}}
\label{tab:mm_selected}
\end{table}

\subsection{Experimental Setups}

\paragraph{Implementations.}

In the cold-start stage, we fine-tune Qwen2.5-VL-Instruct~\citep{bai2025qwen25vl} with curated long CoT data (\S\ref{ssec:cold_start}) using the LLaMA-Factory framework~\citep{zheng2024llamafactory}. Training runs for $5$ epochs with AdamW, a cosine schedule, an initial learning rate of $1\times10^{-5}$, and a $0.1$ warm-up ratio. The checkpoint with the best validation score is retained. This checkpoint initializes the policy for RL training, implemented with the EasyR1 codebase~\citep{zheng2025easyr1}. VAS (Algorithm~\ref{alg:vas_sampler}) is set to $N=32$, $\lambda=0.5$, $\alpha=0.8$, $\beta=0.2$, and $T_{\text{update}}=35$\footnote{Due to the resource constraint, experiments in ablation studies and analysis adopt $N=8$, $\lambda=0.5$, $\alpha=0.5$, $\beta=0.5$, and $T_{\text{update}}=28$ unless otherwise specified.}. Additional hyper-parameters for both cold-start and RL training are detailed in Appendix~\ref{app:hyper}.

\paragraph{Benchmarks.}
To evaluate \ours, we adopt five widely used and challenging benchmarks focusing on mathematical and logical reasoning: MathVerse~\citep{zhang2024mathverse}, MathVista~\citep{lu2024mathvista}, MathVision~\citep{wang2024measuring}, LogicVista~\citep{xiao2024logicvista}, and ChartQA~\citep{masry2022chartqa}. These benchmarks collectively assess diverse aspects of problem-solving, including complex multi-step mathematics, visual reasoning, logical deduction, and chart-based understanding.

\paragraph{Baselines.}
In this work, we compare \ours against a broad set of MLLMs, covering both general-purpose and reasoning-oriented designs of comparable model size. \textbf{General-purpose MLLMs}: Qwen2.5-VL-Instruct-7B~\citep{bai2025qwen25vl}, InternVL2.5-8B~\citep{chen2025internvl25}, InternVL3-8B~\citep{zhu2025internvl3}, and LLaVA-OneVision-7B (LLaVA-OV; \citet{li2024llava}), representing the recent state-of-the-art general-purpose MLLMs. \textbf{Reasoning-oriented MLLMs}: VL-Cogito-7B~\citep{Yuan2025VLCogitoPC}, MM-Eureka-8B~\citep{meng2025mm}, R1-VL-7B~\citep{zhang2025r1}, R1-OneVision-7B~\citep{yang2025r1}, OpenVLThinker-7B~\citep{deng2025openvlthinker}, Vision-R1-7B~\citep{huang2025vision}, VL-Rethinker~\citep{wang2025vl}, and ThinkLite-VL-7B~\citep{wang2025sota}.

\paragraph{Evaluation.}
We adopt a unified prompt across all evaluations, requiring models to enclose final answers in ``\texttt{\textbackslash box\{\}}'' (full prompt in Appendix~\ref{apdx:sys:prompt}). Inference is performed using vLLM~\citep{kwon2023efficient} for efficient generation. For benchmarks with official protocols (\eg MathVision, MMMU), we strictly follow the original procedures. For others, mathematical questions are assessed with Math-Verify~\citep{mathverify} and MathRuler~\citep{mathruler}, while non-mathematical ones use exact matching. To ensure robustness, we further (1) select the most semantically similar option when multiple-choice answers do not exactly match any candidate, and (2) employ GPT-4o~\citep{openai2024gpt4o} as an auxiliary judge for open-ended questions where exact matching or extraction fails.

\subsection{Main Results}
As shown in Table~\ref{tab:mm_selected}, our 7B model achieves state-of-the-art performance among reasoning-oriented MLLMs, reaching an average score of \textbf{58.4}, the highest across all evaluated models. It ranks first on most benchmarks, including \textbf{MathVerse} (55.4), \textbf{MathVision} (31.8), \textbf{LogicVista} (48.9), and \textbf{ChartQA} (83.7), while also delivering competitive results on \textbf{MathVista} (72.0). Compared to other general models with similar scales, our approach consistently yields superior results. 

In addition, the 3B variant of our model demonstrates strong competitiveness with an average of 52.7. Despite its smaller scale, it matches or surpasses several 7B models (e.g., R1-VL at 47.7 and R1-OneVision at 47.8), underscoring the efficiency of our framework in resource-constrained settings. 

These results highlight the complementary contributions of our pipeline: carefully curated long CoT supervision for cold-start initialization, reinforcement learning to incentivize deeper reasoning, and Variance-Aware Sampling (VAS) to stabilize optimization. Together, they enable consistent improvements in reasoning performance across both small- and large-scale models.

\subsection{Effect of Cold-Start and VAS}
\begin{table}[t]
\centering                
\renewcommand{\arraystretch}{1.2}  
\setlength{\tabcolsep}{1.8mm}      
\footnotesize         
\caption{The effect of Cold-start SFT and Variance-Aware Sampling (VAS).}
\adjustbox{max width=1.0\textwidth}{
\begin{tabular}{l ccccc c}
\toprule
\textbf{Model} & \textbf{MathVerse} & \textbf{MathVista} & \textbf{MathVision} & \textbf{LogicVista} & \textbf{ChartQA} & \textbf{Avg} \\
\midrule
Qwen2.5-VL-3B & 40.4 & 63.5 & 24.3 & 38.4 & 76.8 & 48.7 \\
\rowcolor[HTML]{F5F5F5} +Cold-start & 42.1 & 58.0 & 25.2 & 39.5 & 78.5 & 48.7 \\
\rowcolor[HTML]{F5F5F5} +GRPO & 46.2 & 65.6 & 24.7 & 42.4 & 79.9 & 51.8 \\
\rowcolor[HTML]{E8F5E9} +VAS (\ours) & 47.9 & 67.1 & 25.4 & 43.1 & 81.2 & 52.9 \\
\bottomrule
\end{tabular}}
\label{tab:math_abla_cold_vas}
\end{table}

Table~\ref{tab:math_abla_cold_vas} reports the effect of cold-start supervision and the proposed VAS strategy on Qwen2.5-VL-3B across several mathematics-related benchmarks. Beginning with the base model, cold-start fine-tuning on curated long-form CoT data yields consistent improvements, particularly on MathVerse and ChartQA. 
Building upon this, GRPO further enhances performance, surpassing the cold-start baseline and demonstrating that reinforcement learning effectively incentivizes exploration and strengthens reasoning.
The introduction of VAS (MMR1) achieves the highest overall scores, delivering notable gains on MathVerse, MathVista, and LogicVista.
Collectively, these results highlight the complementary roles of different components: (1) cold-start supervision provides a strong initialization by imitating high-quality reasoning trajectories; (2) reinforcement learning emphasizes exploratory behavior to further incentivize reasoning; and (3) VAS ensures stable, variance-aware training, thereby leading to more robust and effective learning outcomes.

\subsection{Effect of VAS Hyper-parameters}
We further investigate the sensitivity of VAS to its key hyperparameters, \ie the mixture ratio $\lambda$, VPS update frequency $T_{\text{update}}$, number of rollouts $N$, and the weighting between OVS and TDS in VPS computation. The results are summarized in Table~\ref{tab:vas_hyper}.

\paragraph{Mixture ratio.} 
The mixture ratio $\lambda$ controls the balance between VPS-weighted sampling and uniform random sampling in Algorithm~\ref{alg:vas_sampler} (Lines 12–14). Performance is generally robust across settings, with $\lambda=0.5$ yielding competitive and stable results. Extremely large ratios (\eg $\lambda=1.0$) reduce coverage of the overall dataset and lead to degraded performance, confirming the necessity of maintaining a balance between variance promotion and coverage.

\paragraph{Update frequency.} 
The update interval $T_{\text{update}}$ specifies the frequency at which VPS scores are refreshed. Short intervals (\eg 4 or 7 steps) enhance adaptability to model dynamics but incur higher computational overhead. Moderate intervals (\ie 14–35 steps) strike a favorable balance, consistently yielding robust performance. In contrast, excessively long intervals (\eg 56 steps) result in outdated VPS estimates, thereby weakening gradient signals and degrading overall performance.

\paragraph{Number of rollouts.} 
The rollout number $N$ affects the accuracy of variance estimation for OVS and TDS. Increasing $N$ from 8 to 16 provides marginal improvements by reducing sampling noise. However, further increases (\eg 32) offer limited gains while introducing higher computational costs.

\paragraph{VPS weighting.} 
The VPS ratio $(\alpha, \beta)$ balances outcome variance (OVS) and trajectory diversity (TDS). A balanced combination (\eg $\alpha=0.8$, $\beta=0.2$) consistently delivers strong results, whereas relying solely on one component leads to instability. Consistent with the analysis in \S\ref{sec:theory}, this shows that outcome variance and trajectory diversity play complementary roles: OVS captures correctness variability, while TDS safeguards a lower bound on variance when correctness signals are sparse.

Overall, the results demonstrate that VAS maintains stability across a wide range of hyperparameter settings, with moderate mixture ratios, appropriate update frequencies, and balanced VPS weighting yielding the most consistent performance improvements.

\begin{table}[t]
\centering
\renewcommand{\arraystretch}{1.2}
\setlength{\tabcolsep}{1.8mm}
\footnotesize
\caption{The effect of Variance-Aware Sampling hyper-parameters.}
\adjustbox{max width=1.0\textwidth}{
\begin{tabular}{l c ccccc c} 
\toprule
\textbf{Ablated Param.} & \textbf{Value} & \textbf{MathVerse} & \textbf{MathVista} & \textbf{MathVision} & \textbf{LogicVista} & \textbf{ChartQA} & \textbf{Avg} \\
\midrule
\rowcolor[HTML]{F5F5F5} Mixture ratio & 0.2 & 46.3 & 67.9 & 23.2 & 40.2 & 79.9 & 51.5 \\
\rowcolor[HTML]{F5F5F5}  & 0.5 & 46.1 & 66.4 & 24.8 & 41.7 & 79.4 & 51.7 \\
\rowcolor[HTML]{F5F5F5}  & 0.8 & 46.9 & 64.8 & 24.3 & 43.8 & 79.9 & 52.0 \\
\rowcolor[HTML]{F5F5F5}  & 1.0 & 44.6 & 65.3 & 24.9 & 43.8 & 79.9 & 51.7 \\
\hdashline
Update freq. & 4 & 47.4 & 65.8 & 23.8 & 39.3 & 79.6 & 51.2 \\
& 7 & 46.6 & 65.7 & 24.6 & 41.5 & 78.9 & 51.5 \\
& 14 & 46.7 & 66.9 & 24.2 & 44.6 & 79.2 & 52.3 \\
& 28 & 46.1 & 66.4 & 24.8 & 41.7 & 79.4 & 51.7 \\
& 35 & 47.6 & 66.1 & 24.5 & 42.4 & 80.1 & 52.2 \\
& 56 & 44.5 & 65.5 & 23.9 & 40.2 & 80.2 & 50.9 \\
\hdashline
\rowcolor[HTML]{F5F5F5} \# rollout & 8 & 46.1 & 66.4 & 24.8 & 41.7 & 79.4 & 51.7 \\
\rowcolor[HTML]{F5F5F5}  & 16 & 45.9 & 65.0 & 24.5 & 41.1 & 79.8 & 51.3 \\
\rowcolor[HTML]{F5F5F5}  & 32 & 46.7 & 65.1 & 25.0 & 42.2 & 78.9 & 51.6 \\
\hdashline
VPS ratio & (0.0, 1.0) & 46.8 & 64.4 & 24.3 & 40.9 & 78.8 & 51.0 \\
(OVS, TDS) & (0.2, 0.8) & 46.4 & 65.6 & 25.0 & 40.6 & 79.7 & 51.5 \\
& (0.5, 0.5) & 46.1 & 66.4 & 24.8 & 41.7 & 79.4 & 51.7 \\
& (0.8, 0.2) & 46.9 & 66.8 & 23.7 & 45.1 & 79.0 & 52.3 \\
& (1.0, 0.0) & 46.5 & 65.6 & 24.5 & 39.7 & 79.6 & 51.2 \\
\bottomrule
\end{tabular}
}
\label{tab:vas_hyper}
\end{table}

\subsection{Efficiency of VAS}

\begin{figure}[t]
    \centering
    \includegraphics[width=0.98\linewidth]{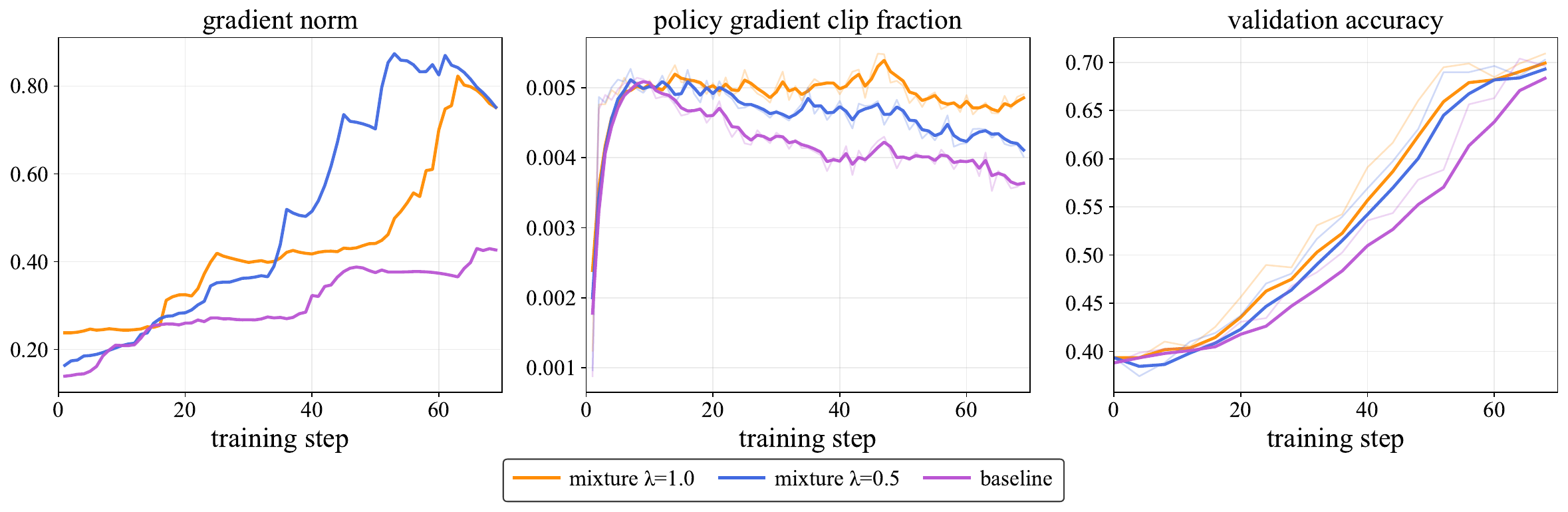}
    \caption{Training efficiency of Variance-Aware Sampling (VAS). 
The plots compare three settings: full VAS sampling ($\lambda=1.0$, \textcolor{orange}{orange}), mixed sampling with half VAS and half random ($\lambda=0.5$, \textcolor{blue}{blue}), and the vanilla baseline (\textcolor{purple}{purple}). 
\textbf{Left:} Actor gradient norm, reflecting the magnitude of gradient signals during training. 
\textbf{Middle:} Policy gradient clip fraction, indicating the proportion of updates reaching the clipping boundary. 
\textbf{Right:} Validation accuracy, showing convergence speed and final performance.}

    \label{fig:training_efficiency_vas}
\end{figure}

Figure~\ref{fig:training_efficiency_vas} presents a comparative analysis of the training efficiency of VAS under different mixture ratios (\textcolor{orange}{orange line}: $\lambda = 1.0$, \textcolor{blue}{blue line}: $\lambda = 0.5$) against the vanilla random-shuffle baseline (\textcolor{purple}{purple line}). The evaluation considers three key indicators: \textit{actor gradient norm}, \textit{policy-gradient clipping fraction}, and \textit{validation accuracy}.

\paragraph{Gradient norm.}  
The gradient norm reflects the overall magnitude of parameter updates. Models trained with VAS consistently exhibit higher gradient norms compared to the shuffle baseline, indicating more substantial and informative updates. This empirical finding is consistent with our theoretical analysis in \S\ref{sec:theory}, which demonstrates that prompts with higher reward variance produce gradients characterized by greater magnitude and more reliable signal strength.

\paragraph{Policy-gradient clip fraction.}  
The clip fraction quantifies the \textit{frequency} with which the policy update magnitude reaches the clipping threshold in GRPO. Within stable ranges, a higher clip fraction indicates more effective learning: the model performs substantial yet constrained updates, exploits the trust region more fully, and extracts stronger signals from each batch.
As illustrated in Figure~\ref{fig:training_efficiency_vas}, VAS configurations achieve higher and more stable clip fractions compared to the baseline, highlighting their improved sample efficiency and more effective exploration of the policy space.  

\paragraph{Validation accuracy.}  
On the held-out validation set, VAS demonstrates consistent improvements in both convergence speed and final accuracy compared to the shuffle baseline.
Employing full VAS sampling (\(\lambda = 1.0\)) yields the most rapid and stable performance gain, while the mixed configuration (\(\lambda = 0.5\)) also surpasses uniform sampling in convergence efficiency. Although the difference between \(\lambda = 1.0\) and \(\lambda = 0.5\) is not pronounced on this mathematics-focused set, it is worth noting that incorporating partial random sampling can, in principle, encourage broader data coverage and reduce the risk of oversampling a limited subset of prompts. This trade-off is expected to be more beneficial in domains characterized by greater content heterogeneity or less structured reward signals.




\subsection{Dynamics of Variance Promotion Score}

Figure~\ref{fig:curr_weight_transfer} summarizes how Variance Promotion Scores (VPS) evolve across update intervals ($t\!\to\!t{+}14$), where the \textit{histograms} show the marginal VPS distribution and \textit{transition matrices} of VPS assignments between step $t$ (rows) and step $t{+}14$ (columns)\footnote{$T_{\text{update}}$ is set to 14.}.

\paragraph{Convergence of rankings.}
As training progresses, the mass in the transition matrices progressively concentrates along the diagonal, while the off–diagonal dispersion diminishes. This behavior reflects the progressive stabilization of per-prompt VPS rankings, converging toward a consistent ordering. Such convergence aligns with the intended steady state of VAS, in which high-variance ``frontier'' prompts persist only as a comparatively small yet stable subset.

\paragraph{Bidirectional movement (adaptation).}
Non-negligible amount of off–diagonal mass persists in both directions (low$\!\to$high and high$\!\to$low), reflecting the adaptive nature of VPS under the evolving policy. Certain prompts gain informativeness as their outcome become more balanced or their trajectory patterns diversity increase, whereas others gradually lose informativeness once they are either consistently solved with ease or repeatedly lead to failure.

\paragraph{High$\!\to$high persistence fades.}
In the early stages, a visible block appears in the bottom-right region (high$\!\to$high), but this gradually weakens as training progresses. This pattern indicates that many prompts are initially challenging, remaining near the maximum OVS/TDS levels across updates. As the policy improves, however, these prompts either become polarized in correctness—lowering OVS—or converge toward more uniform trajectories—lowering TDS. Consequently, they gradually exit the high-VPS subset.

\begin{figure}[t]
    \centering
    \includegraphics[width=\linewidth]{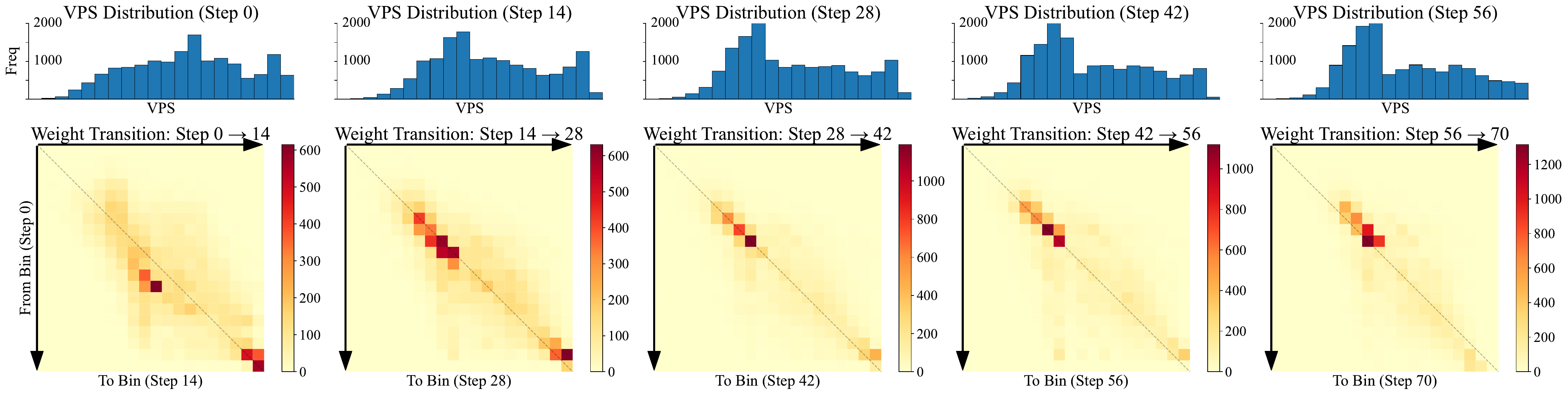}
    \caption{Dynamics of Variance Promotion Score (VPS) during training. 
The top row illustrates the distribution of VPS values across data points at different training steps. 
The bottom row shows transition matrices of VPS assignments between consecutive update intervals, where each cell indicates the number of data points moving from a source bin (vertical axis) at the earlier step to a target bin (horizontal axis) at the later step. 
The arrows indicate the direction from lower to higher VPS bins, facilitating interpretation of upward or downward transitions.}
    \label{fig:curr_weight_transfer}
\end{figure}

\paragraph{Distributional shift toward mid-VPS.}
The marginal distribution of VPS progressively evolves from a relatively flat shape with a pronounced upper tail to a more compact form centered around medium scores. Since $\mathrm{VPS}=\alpha\,\mathrm{OVS}+\beta\,\mathrm{TDS}$ with $\alpha(0.8)>\beta(0.2)$, this shift reflects two key dynamics: (1) as training proceeds, fewer prompts achieve near-maximum OVS (with pass rates $\approx0.5$), and (2) the contribution of TDS establishes a residual floor, anchoring many prompts at moderate VPS even when correctness becomes less stable. This emergent clustering around mid-VPS characterizes a signature of convergence: VAS increasingly narrows its focus to a frontier of prompts where reward variance remains informative and useful for further optimization.

\paragraph{Asymmetry of flows.}
In later intervals, transitions from higher to mid VPS bins occur more frequently than movements in the opposite direction, implying a gradual reduction in reward variance as competence improves. This observation aligns with the Variance–Progress principle (ref. \S\ref{sec:theory}): once learning progress has been achieved on a given prompt, both its reward variance and VPS tend to diminish, thereby encouraging the sampler to redistribute probability mass toward other frontier items that remain more uncertain or informative.

\subsection{Qualitative Demonstration}
\begin{figure}[t]
    \centering
    \includegraphics[width=\linewidth]{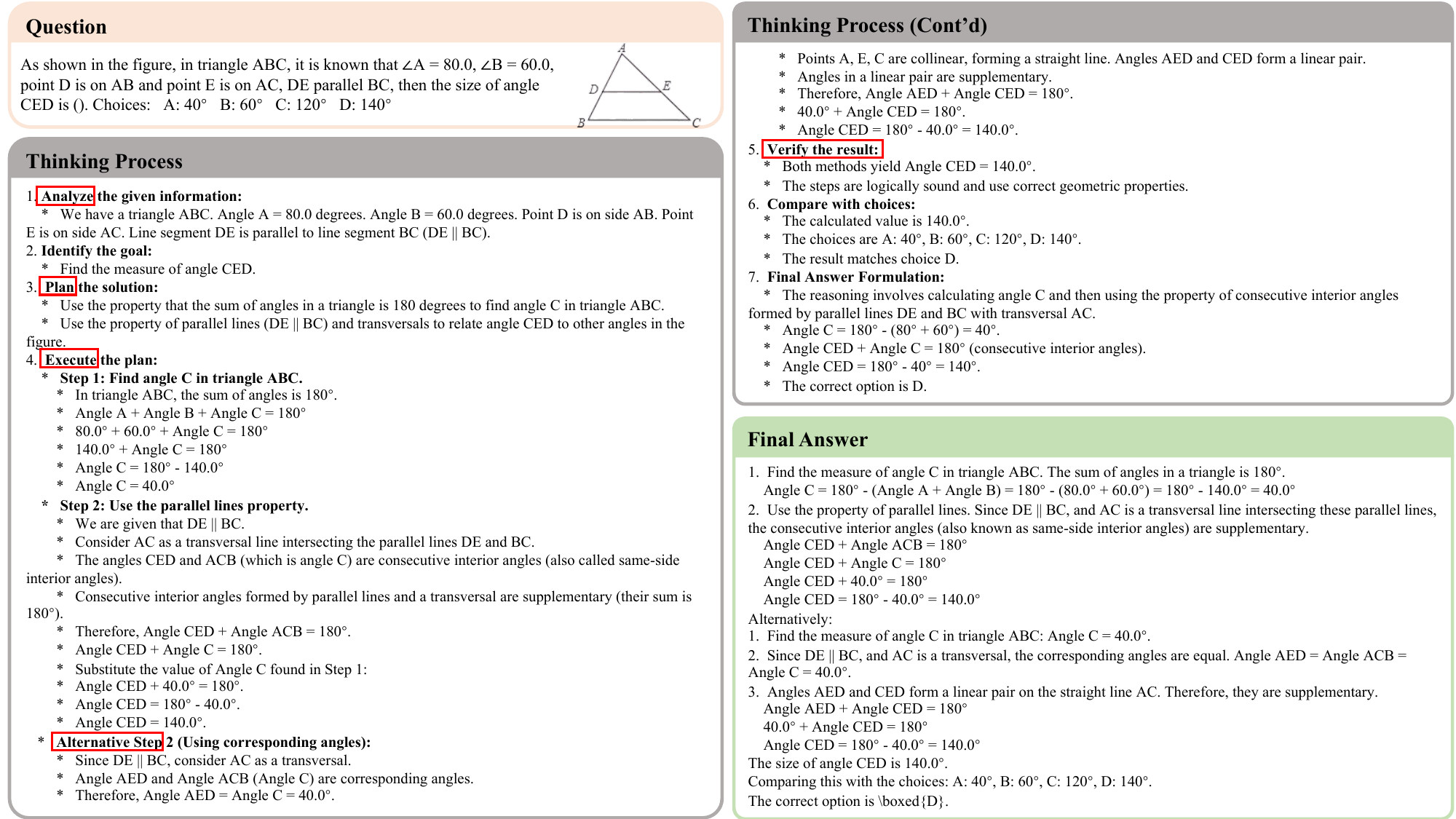}
    \caption{Qualitative demonstration of MMR1’s reasoning process on a MathVerse problem. The figure illustrates the input question, the model’s step-by-step thinking process, and the final answer. The reasoning is logically structured, including problem analysis, solution planning, execution, verification, and alternative approaches, ultimately arriving at the correct answer ($140^\circ$).}
    \label{fig:demo}
    \vspace{-2mm}
\end{figure}

As shown in Figure~\ref{fig:demo}, this MathVerse question highlights the reasoning capability of \ours. The solution it generates follows a clear and logical structure: the model begins by restating the given conditions, then applies the angle-sum property of a triangle to determine the missing angle, and finally uses parallel-line properties to compute the target angle. This step-by-step organization reflects a coherent “analyze–plan–execute” process.
The response also demonstrates reflective reasoning. After deriving the result, the model verifies consistency with geometric rules and cross-checks against the answer choices. Furthermore, it provides an alternative method based on corresponding and supplementary angles, which strengthens confidence in the correctness of the conclusion.

Overall, the model not only produces the correct answer but also exhibits robust reasoning behaviors, including systematic decomposition, verification, and multiple-solution perspectives, which illustrate strong problem-solving ability beyond direct computation.

\section{Conclusion and Limitation}
In this work, we investigate the challenge of gradient vanishing in reinforcement learning for multimodal reasoning. We introduce Variance-Aware Sampling (VAS), a sampling strategy that exploits outcome variance and trajectory diversity to prioritize informative prompts while maintaining broad data coverage. Grounded in theoretical analysis and supported by extensive empirical evaluation, VAS enhances training stability and improves the effectiveness of reinforcement learning in multimodal reasoning tasks. Beyond the methodological contribution, a central outcome of this work is the release of large-scale, carefully curated cold-start datasets and well-tuned models, which we hope will provide valuable resources for benchmarking and advancing future research in this area.

Despite these contributions, our work has several limitations. First, although VAS mitigates gradient vanishing, it does not fully resolve all training instabilities inherent to multimodal reinforcement learning. Second, 
the computation of variance-based prompt scores (VPS) incurs additional overhead, though this can be mitigated by increasing update intervals or selectively updating a subset of samples. 
Finally, 
our method primarily focuses on data sampling; while it is expected to complement algorithmic advances in reinforcement learning, a systematic investigation into their integration is left to future work.

Looking ahead, we believe this study opens several promising avenues. Future research may explore extending VAS to broader domains, examining its interaction with diverse reward designs, and integrating it with more advanced reinforcement learning algorithms to further improve sample efficiency and robustness. 
We hope that our methodological innovations, together with the released resources, will provide a solid foundation for the community to advance the development of more stable and capable multimodal reasoning models.

\newpage

\bibliography{iclr2026_conference}
\bibliographystyle{iclr2026_conference}

\clearpage
\appendix

\section{Variance--Progress Theory}\label{apdx:sec:proofs}

\begin{tcolorbox}[colback=gray!5,colframe=gray!35,boxrule=0.4pt,arc=1.2mm,left=4pt,right=4pt,top=3pt,bottom=3pt]
\footnotesize
\textbf{Intuition.} Gradient updates are informative only if sampled rollouts produce \emph{different} rewards. 
When rewards are nearly identical, advantages collapse and gradients vanish. 
Our theorem shows that (under standard smoothness and non-degeneracy conditions) the expected improvement after a single update is \emph{linearly} lower-bounded by the reward variance of the prompt. 
Thus, selecting prompts that induce mixed outcomes (OVS) and diverse reasoning paths (TDS) provably strengthens learning.
\end{tcolorbox}

\vspace{0.25em}
\noindent\textit{At a glance (notation \& assumptions).} 
We write $g(x,y)=\nabla_\theta\log\pi_\theta(y\mid x)$, $\bar R(x)=\mathbb E_y[R(x,y)]$, $R_{\mathrm{res}}=R-\bar R$, and $\Gamma_\theta(x)=\mathbb E_y[g\,g^\top]$. 
Assumptions: (i) $L$-smoothness of $J_x$; (ii) bounded score function $\|g\|\!\le\!G_{\max}$; (iii) uniform positive definiteness of $\Gamma_\theta(x)\succeq \lambda_{\min}I$; (iv) gradient lower bound $\|\nabla J_x\|^2\!\ge\!c_{\min}\operatorname{Var}[R]$ (a mild consequence of (iii)).

\subsection{Technical Preliminaries}\label{apdx:ssec:lemma-var}
Throughout we fix a prompt $x\in\mathcal X$ and write expectations over $y$ as 
$\mathbb E_{y}[\,\cdot\,]=\mathbb E_{y\sim\pi_\theta(\cdot\mid x)}[\,\cdot\,]$.
The policy is differentiable and strictly positive on~$\mathcal Y$.  
We denote $\bar R(x)=\mathbb E_y[R(x,y)]$ and $R_{\mathrm{res}}(x,y)=R(x,y)-\bar R(x)$.

\subsubsection{Action-independent baselines}
\begin{proposition}[Optimal action-independent baseline]
\label{prop:optimal-baseline}
For any square-integrable reward $R$, the baseline $b^{\star}(x)=\bar R(x)$ minimizes the total variance of the REINFORCE gradient estimator over all baselines $b(x)$ that depend on $x$ but not on $y$.  
\end{proposition}

\begin{proof}
Let $G_b(x)=\mathbb E_y[g(x,y)(R(x,y)-b(x))]$, where $g(x,y)=\nabla_\theta\log\pi_\theta(y\mid x)$.  
Using $\mathbb E_y[g]=0$, the covariance term vanishes and
$\operatorname{Var}[G_b(x)]=\mathbb E_y[\|g\|^{2}(R-b)^{2}]$.
The right-hand side is a convex quadratic in $b$. Differentiating and setting to zero yields $b(x)=\bar R(x)$ \citep{williams1992reinforce}.
\end{proof}

\paragraph{Remark.}
An action-dependent baseline such as 
$b^{\|g\|^2}(x)={\mathbb E_y[\|g\|^{2}R]}/{\mathbb E_y[\|g\|^{2}]}$  
can further reduce scalar variance~\citep{Greensmith2001VarianceRT}, but requires inner Monte-Carlo estimates. The bounds below hold for any $y$-independent baseline.

\subsubsection{Gradient-variance bounds}
With $b^\star(x)=\bar R(x)$,
\[
    G(x)=\mathbb E_y[g(x,y)R_{\mathrm{res}}(x,y)], 
    \qquad
    \mathbb E[G(x)]=\nabla_\theta J_x(\theta).
\]
Its covariance is
\[
    \operatorname{Var}[G(x)]
      =\mathbb E_y[R_{\mathrm{res}}^{2}\,g\,g^{\!\top}]
       -\nabla_\theta J_x(\theta)\nabla_\theta J_x(\theta)^{\!\top}.
\]
For bounding we drop the nonnegative outer product, which only reduces the variance, yielding a valid but looser inequality.

\begin{assumption}[Bounded score-function gradient]\label{as:bounded-grad}
There exists $G_{\max}<\infty$ such that $\|g(x,y)\|\le G_{\max}$ for all $(x,y)$.  
\end{assumption}

\begin{lemma}[Variance sandwich bound]\label{lem:var-bounds}
Under Assumption~\ref{as:bounded-grad}, let $\Gamma_\theta(x)=\mathbb E_y[g(x,y)g(x,y)^{\!\top}]$. If $\lambda_{\min}>0$ is the smallest eigenvalue of $\Gamma_\theta(x)$ uniformly over~$x$, then
\[
    \lambda_{\min}\,\operatorname{Var}_y[R(x,y)]\,I_d
    \;\preceq\;
    \operatorname{Var}[G(x)]
    \;\preceq\;
    G_{\max}^{2}\,\operatorname{Var}_y[R(x,y)]\,I_d.
\]
Thus reward variance is the only prompt-dependent factor; Fisher terms contribute bounded, model-dependent constants.
\end{lemma}

\begin{proof}
For any unit vector $v$,
$v^{\!\top}\operatorname{Var}[G(x)]v
=\mathbb E_y[(v^{\!\top}g)^{2}R_{\mathrm{res}}^{2}]$.
Upper bound: $(v^{\!\top}g)^2 \le \|g\|^2 \le G_{\max}^2$.  
Lower bound: $\Gamma_\theta(x)\succeq\lambda_{\min}I$ implies  
$v^{\!\top}\operatorname{Var}[G(x)]v \ge \lambda_{\min}\mathbb E_y[R_{\mathrm{res}}^2] = \lambda_{\min}\operatorname{Var}_y[R]$.
\end{proof}

\subsubsection{Bounds on the Fisher term}
\begin{proposition}[Uniform Fisher bounds]\label{prop:fisher-bounds}
Under Assumption~\ref{as:bounded-grad}, there exist constants $0<\lambda_{\min}\le\lambda_{\max}=G_{\max}^2$ such that
\[
    \lambda_{\min}\,I_d
    \;\preceq\;
    \Gamma_\theta(x)
    \;\preceq\;
    \lambda_{\max}\,I_d,
    \qquad \forall x\in\mathcal X.
\]
\end{proposition}

\begin{proof}
The upper bound follows from $\|g\|\le G_{\max}$. The lower bound holds under the standard non-degeneracy assumption that $\pi_\theta(\cdot\mid x)$ defines a full-dimensional exponential family and $\theta$ ranges over a compact set, ensuring $\Gamma_\theta(x)\succ0$ uniformly \citep{amari2000methods}.
\end{proof}

\subsection{Proof of the Variance--Progress Theorem}\label{apdx:ssec:theorem-proof}

Let the update be $\theta^{+}=\theta+\eta G(x)$ with $\eta>0$.  
A second-order Taylor expansion yields
\[
J_x(\theta^{+})-J_x(\theta)
   = \eta\langle\nabla J_x, G(x)\rangle
   + \tfrac12 \eta^2 G(x)^{\!\top}H_x(\tilde\theta)G(x),
\]
for some $\tilde\theta$ on the segment $[\theta,\theta^{+}]$. Taking expectations and $L$-smoothness,
\[
\mathbb E[J_x(\theta^{+})-J_x(\theta)]
   \ge \eta\|\nabla J_x\|^2
   - \tfrac12 \eta^2 L \,\mathbb E[\|G(x)\|^2].
\]
Decomposing $\mathbb E[\|G(x)\|^2]$ into bias and variance and using Lemma~\ref{lem:var-bounds},
\[
\mathbb E[J_x(\theta^{+})-J_x(\theta)]
   \ge \eta\|\nabla J_x\|^2
   - \tfrac12 \eta^2 L \bigl(\|\nabla J_x\|^2 + dG_{\max}^2\operatorname{Var}[R]\bigr).
\]
Assuming $\|\nabla J_x\|^2 \ge c_{\min}\operatorname{Var}[R]$,
\[
\mathbb E[J_x(\theta^{+})-J_x(\theta)]
   \ge \eta c_{\min}\operatorname{Var}[R]
     - \tfrac12 \eta^2 L(c_{\min}+dG_{\max}^2)\operatorname{Var}[R].
\]
Choosing $\eta \le \tfrac{c_{\min}}{2L(c_{\min}+dG_{\max}^2)}$ gives
$\mathbb E[J_x(\theta^{+})-J_x(\theta)] \ge \tfrac{\eta c_{\min}}{2}\operatorname{Var}[R]$.
Restricting further to $\eta \le c_{\min}/(4L)$ yields the simplified bound used in the main text:
\[
\mathbb E[J_x(\theta^{+})-J_x(\theta)]
   \ge \tfrac{\eta c_{\min}}{4}\operatorname{Var}_{y\sim\pi_\theta}[R(x,y)].
\]

\paragraph{Discussion.} The bound depends only on constants $c_{\min}$ and $L$ tied to the model family. All prompt dependence enters through $\operatorname{Var}[R]$, establishing reward variance as the decisive quantity.

\subsection{Two-Level Decomposition of Reward Variance}\label{apdx:ssec:two-level-decomp}

Let $y=(y_{\mathrm{cot}},y_{\mathrm{ans}})$ and define $R(x,y)=\mathbf 1_{\text{verifier}(y_{\mathrm{ans}})}$.  
Let $Z=\varphi(y_{\mathrm{cot}})$ represent the chain. By the law of total variance,
\[
\operatorname{Var}_y[R]
   = \mathbb E_Z[p_Z(1-p_Z)] + \operatorname{Var}_Z[p_Z],
\]
where $p_Z(x)=\Pr(R=1\mid Z)$.

\paragraph{Intra/inter-trajectory terms.}  
The first term is intra-trajectory Bernoulli variance; the second is inter-trajectory variation of success probabilities.  

\paragraph{Efron--Stein lower bound.}  
If $|p_z - p_{z'}|\le L\,d(y_{\mathrm{cot}}(z),y_{\mathrm{cot}}(z'))$ for a bounded distance $d\in[0,1]$, then
\[
\operatorname{Var}_Z[p_Z] \ge \tfrac{L^2}{4}\, \mathbb E_{Z,Z'}[d^2(y_{\mathrm{cot}}(Z),y_{\mathrm{cot}}(Z'))].
\]

\paragraph{OVS and TDS estimators.}  
With $K$ rollouts, $\hat p=\tfrac1K\sum R$,  
\[
\widehat{\text{OVS}}(x)=\hat p(1-\hat p), \qquad
\text{TDS}(x)=\tfrac{1}{K(K-1)}\sum_{i\neq j} d^2(y^{(i)}_{\mathrm{cot}},y^{(j)}_{\mathrm{cot}}).
\]
By the strong law and U-statistic convergence, both estimators are strongly consistent for their respective population terms.  

\paragraph{Variance Promotion Score (VPS).}  
Define $\widehat{\text{VPS}}=\alpha\widehat{\text{OVS}}+\beta\text{TDS}$ with $\alpha,\beta>0$.  
Then $\widehat{\text{VPS}}$ converges almost surely to a positive affine transform of a \emph{lower bound} on $\operatorname{Var}[R]$. Hence, VPS is a strongly consistent monotone surrogate for reward variance (it does not require equality to hold).

\subsection{Extension from REINFORCE to GRPO}\label{apdx:ssec:grpo}
GRPO replaces the scalar baseline by the in-batch mean reward and whitens with the sample standard deviation:
\[
\tilde R(x,y_i) = \frac{R(x,y_i)-\bar r(x)}{s(x)+\delta}.
\]
This yields a centered, variance-controlled REINFORCE estimator. Multiplying by importance ratios (and optionally clipping) preserves the core dependence on reward variance: the Variance--Progress lower bound continues to hold after a rescaling of constants. When clipping is enabled, an $O(\varepsilon)$ bias arises; the bound is reduced by the same order but remains strictly positive whenever $\operatorname{Var}[R]>0$. Thus prompts that induce higher reward variance guarantee larger expected improvements under GRPO, paralleling vanilla REINFORCE.

\section{Fine-Grained Math Type Definitions}
\label{apdx:ssec:math_types}

For the construction of our math dataset, each problem is assigned to one of thirteen fine-grained categories.
These categories provide balanced coverage across fundamental and advanced domains of mathematics, and ensure that the RL dataset spans diverse reasoning skills.
The formal definitions, explanations, and illustrative examples for each category are presented below.

\begin{enumerate}
    \item \textbf{Arithmetic} \\
    \textit{Definition:} Arithmetic covers basic numerical operations, including addition, subtraction, multiplication, and division. \\
    \textit{Explanation:} It forms the foundation of mathematics by establishing rules for manipulating numbers. \\
    \textit{Example:} Compute $15 \times 12$.

    \item \textbf{Counting} \\
    \textit{Definition:} Counting addresses enumeration of objects or elements within a collection. \\
    \textit{Explanation:} It includes both simple enumeration and principles such as permutations and combinations. \\
    \textit{Example:} Determine how many integers between 1 and 100 are multiples of 5.

    \item \textbf{Combinatorics} \\
    \textit{Definition:} Combinatorics studies arrangements, selections, and combinations of discrete objects. \\
    \textit{Explanation:} It extends counting principles to complex scenarios involving structured sets. \\
    \textit{Example:} How many ways can 5 distinct books be arranged in a row?

    \item \textbf{Algebra} \\
    \textit{Definition:} Algebra represents relationships using symbols and equations. \\
    \textit{Explanation:} It provides systematic tools for solving equations and reasoning about unknowns. \\
    \textit{Example:} Solve $2x + 3 = 9$ for $x$.

    \item \textbf{Functions} \\
    \textit{Definition:} A function maps each input to exactly one output. \\
    \textit{Explanation:} Functions formalize dependencies between quantities, described via formulas, graphs, or rules. \\
    \textit{Example:} Given $f(x) = x^2 + 3$, find $f(2)$.

    \item \textbf{Plane Geometry} \\
    \textit{Definition:} Plane Geometry studies figures such as lines, angles, and polygons in two dimensions. \\
    \textit{Explanation:} It addresses lengths, angles, and areas of flat figures. \\
    \textit{Example:} Find the area of a triangle with base $10$ and height $5$.

    \item \textbf{Solid Geometry} \\
    \textit{Definition:} Solid Geometry extends geometric reasoning to three-dimensional figures. \\
    \textit{Explanation:} It concerns volumes, surface areas, and properties of 3D objects. \\
    \textit{Example:} Find the volume of a cube with side length $4$.

    \item \textbf{Combinatorial Geometry} \\
    \textit{Definition:} Combinatorial Geometry analyzes discrete configurations of geometric objects. \\
    \textit{Explanation:} It merges counting with geometry, such as enumerating diagonals or intersections. \\
    \textit{Example:} How many diagonals does a convex octagon have?

    \item \textbf{Descriptive Geometry} \\
    \textit{Definition:} Descriptive Geometry represents 3D objects using 2D projections. \\
    \textit{Explanation:} It enables precise measurement of spatial relationships via orthographic or perspective drawings. \\
    \textit{Example:} Sketch the top and front views of a cube resting on a horizontal plane.

    \item \textbf{Graph Theory} \\
    \textit{Definition:} Graph Theory studies structures composed of vertices and edges. \\
    \textit{Explanation:} It focuses on connectivity, paths, and cycles in discrete networks. \\
    \textit{Example:} Given a graph with vertices A, B, C, D and edges AB, BC, CD, and DA, does the graph contain a cycle?

    \item \textbf{Logic} \\
    \textit{Definition:} Logic studies formal reasoning and inference. \\
    \textit{Explanation:} It examines propositions, truth values, and the validity of conclusions. \\
    \textit{Example:} If "All cats are mammals" and "Fluffy is a cat," deduce whether "Fluffy is a mammal."

    \item \textbf{Statistics} \\
    \textit{Definition:} Statistics concerns the collection, analysis, and interpretation of data. \\
    \textit{Explanation:} It uses measures such as mean, median, and variance to summarize data. \\
    \textit{Example:} For $\{2, 4, 4, 6, 8\}$, compute the mean and median.

    \item \textbf{Topology} \\
    \textit{Definition:} Topology studies properties of spaces invariant under continuous deformations. \\
    \textit{Explanation:} It focuses on qualitative properties such as connectivity and the number of holes. \\
    \textit{Example:} Explain why a doughnut (torus) cannot be deformed into a sphere without removing the hole.
\end{enumerate}

These categories are used to uniformly sample math problems in the RL stage, ensuring both difficulty balance and coverage across diverse mathematical skills.

\section{Detailed Hyper-parameters}
\label{app:hyper}
The main hyperparameters used in the cold-start supervised fine-tuning stage are summarized in Table~\ref{tab:cold_hyp}.

\begin{table}[ht]
\centering
\caption{Cold-start stage hyperparameters.}
\label{tab:cold_hyp}
\begin{tabular}{ll}
\toprule
\textbf{Setting} & \textbf{Value} \\
\midrule
Training epochs & 3 \\
Gradient accumulation steps & 4 \\
Effective batch size & 32  \\
Sequence length cutoff & 16{,}384 \\
Optimizer & AdamW \\
Learning rate & $1\times10^{-5}$ \\
Learning rate schedule & Cosine decay \\
Warm-up ratio & 0.1 \\
Weight decay & 0 \\
Max gradient norm & 1.0 \\
Precision & bfloat16 \\
Deepspeed config & ZeRO-3 (32 shards) \\
\bottomrule
\end{tabular}
\end{table}

The main hyperparameters used in the RL stage (GRPO with VAS) are summarized in Table~\ref{tab:rl_hyp}.

\begin{table}[ht]
\centering
\caption{RL stage hyperparameters.}
\label{tab:rl_hyp}
\begin{tabular}{ll}
\toprule
\textbf{Setting} & \textbf{Value} \\
\midrule
Initialization (policy) & Cold-start checkpoint (Qwen2.5-VL) \\
Objective & GRPO (adv\_estimator=\texttt{grpo}) \\
Reward & $1*\text{format reward} + 1*\text{accuracy reward}$ \\
KL regularization & Enabled, penalty=\texttt{low\_var\_kl}, coef $=0.01$ (fixed) \\
Precision & bfloat16 \\
System prompt & As in Appx.~\ref{apdx:sys:prompt} \\
\midrule
\multicolumn{2}{l}{\textit{Data \& VAS sampling}} \\
Max prompt / response length & 2048 / 4096 \\
Image resolution limits & $[7{,}056,\;1{,}048{,}576]$ pixels \\
Sampling strategy & VAS (\texttt{curriculum}) \\
VAS metrics / weights & learnability (OVS) 0.8,\; self\_bleu\_123 (TDS) 0.2 \\
VAS update frequency & 56 steps \\
VAS mixture ratio & 0.5 (weighted vs. uniform) \\
VPS rollout for scoring & $n=16$, batch size $=4096$ \\
\midrule
\multicolumn{2}{l}{\textit{Rollout \& inference}} \\
Sampler & vLLM \\
Samples per prompt & $n=8$ \\
Temperature / top-$p$ / top-$k$ & 0.6 / 1.0 / $-1$ \\
Validation override & temp $=0.5$, $n=1$ \\
GPU memory util (vLLM) & 0.75 \\
\midrule
\multicolumn{2}{l}{\textit{Optimization (actor / critic)}} \\
Global batch size (actor / critic) & 512 / 256 \\
Micro-batch (update) (actor / critic) & 8 / 4 \\
Micro-batch (experience) (actor / critic) & 32 / 16 \\
PPO epochs & 20 \\
Clipping & policy clip\_low=0.2, clip\_high=0.2; value cliprange=0.5 \\
Max grad norm & 1.0 (both) \\
Optimizer & AdamW (lr $=1\times10^{-6}$, betas $=(0.9,0.999)$) \\
Weight decay & 0.01 \\
LR warmup ratio & 0.0 (constant warmup style) \\
\midrule
\multicolumn{2}{l}{\textit{Parallelism \& hardware}} \\
FSDP & Full shard (policy \& critic), \texttt{fsdp\_size}=8 \\
Episodes  & total\_episodes $=10$ \\
\bottomrule
\end{tabular}
\end{table}

\section{System Prompt}\label{apdx:sys:prompt}
\textbf{Training Prompt.}  
The following system prompt is used during RL training to enforce a structured reasoning format and to require the final answer to be enclosed in ``\texttt{\textbackslash boxed\{\}}''.  

\begin{table}[ht]
\centering
\caption{System prompt used in RL training.}
\label{tab:rl_prompt}
\begin{tcolorbox}[colframe=gray!60, colback=gray!10, fonttitle=\bfseries, coltitle=black, boxrule=0.6mm, arc=2mm, width=\textwidth, toptitle=3pt, bottomtitle=3pt]
\footnotesize
A conversation between User and Assistant. The User provides an image and asks a question. The Assistant first analyzes both the image and the question, then carefully thinks about the reasoning process step by step, and finally provides the User with an accurate answer. The Assistant must carefully checkout the correctness and validity of each reasoning step. If any errors or inconsistencies are found during the reasoning process, the Assistant reflects and corrects them logically. The reasoning process and answer are enclosed within <think> </think> and <answer> </answer> tags, respectively, i.e., <think> reasoning process here, with potential reflections and corrections </think><answer> final answer here, with the key result enclosed in \textbackslash boxed\{\} </answer>.

\end{tcolorbox}
\end{table}

\textbf{Evaluation Prompt.}  
To ensure fairness and generalizability, we use a simplified prompt for evaluation instead of the training prompt.  

\begin{table}[ht]
\centering
\caption{System prompt used for evaluation.}
\label{tab:eval_prompt}
\begin{tcolorbox}[colframe=gray!60, colback=gray!10, fonttitle=\bfseries, coltitle=black, boxrule=0.6mm, arc=2mm, width=\textwidth, toptitle=3pt, bottomtitle=3pt]
\footnotesize
Please solve the problem step by step and put your answer in one \textbackslash boxed\{\}.  
If it is a multiple-choice question, only one letter should appear inside the \textbackslash boxed\{\}.
\end{tcolorbox}
\end{table}

\end{document}